\documentclass{article}
\usepackage{proceed2e} 

\usepackage{natbib}
\usepackage{amsmath}
\usepackage{amsthm}
\usepackage{clrscode3e}

\usepackage{times}

\usepackage{enumitem}

\title{
	\vspace*{-0.5in}
	{\small \hfill \textit{An original version of this paper appeared in UAI-17}}\\
	\vspace*{.25in}
	Exact Inference for Relational Graphical Models with Interpreted Functions:
	Lifted Probabilistic Inference Modulo Theories
	}
\author{
Rodrigo de Salvo Braz \qquad\qquad Ciaran O'Reilly\\
Artificial Intelligence Center \\
SRI International \\
Menlo Park, CA USA \\
}

\usepackage{amsmath}
\usepackage{amssymb}
\usepackage{dsfont}
\usepackage{xspace}
\usepackage{algorithm}
\usepackage{graphicx}
\usepackage{color}
\usepackage{thmtools, thm-restate}

\newtheorem{definition}{Definition}[section]
\newtheorem{theorem}{Theorem}[section]

\newtheorem{example}[theorem]{Example}

\newcommand{\ifte} [3]{\ensuremath{\mathtt{if}\,#1\,\mathtt{then}\,#2\,\mathtt{else}\,#3}\xspace}

\newcommand{\iftt}  {\ensuremath{\mathtt{if }}\;}
\newcommand{\thentt}{\;\ensuremath{\mathtt{ then }}\;}
\newcommand{\elsett}{\;\ensuremath{\mathtt{ else }}\;}

\newcommand{\true}{\ensuremath{\textsc{true}}\xspace}
\newcommand{\false}{\ensuremath{\textsc{false}}\xspace}

\newcommand{\bx}{\ensuremath{\mathbf{x}}\xspace}

\newcommand{\bq}{\ensuremath{\mathbf{q}}\xspace}
\newcommand{\br}{\ensuremath{\mathbf{r}}\xspace}
\newcommand{\be}{\ensuremath{\mathbf{e}}\xspace}
\newcommand{\bu}{\ensuremath{\mathbf{u}}\xspace}

\newcommand{\bff}{\ensuremath{\mathbf{f}}\xspace}

\newcommand{\nesting}{\ensuremath{N_{\cap \cup}}\xspace}

\newcommand{\comment}[1]{}
\newcommand{\define}[1]{\textbf{#1}}

\newcommand\numberthis{\addtocounter{equation}{1}\tag{\theequation}}

\begin{document}

\maketitle

\begin{abstract}
	Probabilistic Inference Modulo Theories (PIMT) is a recent framework that expands
	exact inference on graphical models to use
	richer languages that include arithmetic, equalities, and inequalities
	on both integers and real numbers.
	In this paper, we expand PIMT to a \emph{lifted} version that
	also processes random functions and relations.
	This enhancement is achieved by adapting \emph{Inversion}, a method from Lifted First-Order Probabilistic Inference literature, to also be modulo theories.
	This results in the first algorithm for exact probabilistic inference
	that efficiently and simultaneously exploits
	random relations and functions, arithmetic, equalities and inequalities.
\end{abstract}

\section{INTRODUCTION}

Graphical models such as Bayesian networks and Markov networks
\citep{pearl88probabilistic}
are a well-principled way of representing probabilistic models.
They represent the dependences between random variables in a factorized form
that can be exploited
by probabilistic inference algorithms (such as Variable Elimination \citep{zhang94simple}
and Belief Propagation \citep{pearl88probabilistic}) for greater efficiency.

However, traditional graphical models representations
are unable to represent other types of useful structures:
\begin{itemize}[leftmargin=0.15in]
	\item \define{Relational structure} occurs
	when families of random variables share the same dependences.
	Suppose a model involves random variables
	$sunny$, $happy(Ann)$, $happy(Bob),\dots$
	and the analogous dependences
	$P(happy(Ann) | sunny)$, $P(happy(Bob)|sunny)$ and so on.
	Intuitively, this structure can be exploited for greater efficiency
	since the same
	inference can often be performed only once for an entire
	family of random variables.
	Traditional graphical model representations cannot explicitly indicate
	this structure and thus algorithms cannot directly exploit it.
	\item \define{Algebraic structure}	occurs when dependences between
	random variables (such as conditional probabilities)
	can be compactly described by an algebraic expression.
	For example, it may be that $P(x \in \{1,\dots,1000\} | y \in \{1,\dots,1000\})$
	is equal to $\ifte{x = y}{0.8}{0.2/999}$.
	Intuitively, this can be exploited for greater efficiency
	because large groups of values can be treated in the same way
	(in this case, all pairs of values $(x,y)$ for which $x \neq y$
	are equally probable).
	Traditional graphical model representations, however,
	represent such functions as a lookup table (a large $1000\times1000$ one for this example)\footnote{But note exceptions in  \cite{csi,sanner2012symbolic}.},
	thus not explicitly representing the algebraic structure,
	which prevents algorithms from directly exploiting it.
\end{itemize}

These types of structure are commonly found in real-world applications
and fields such as probabilistic programming,
so research has been conducted on exploiting both:
\begin{itemize}[leftmargin=0.15in]
	\item \define{Lifted probabilistic inference}
	\citep{poole03first,desalvobraz07thesis,milch08lifted,vanDenBroeck11lifted,kersting12lifted}
	exploits \define{relational structure}
	that is explicitly represented with richer languages
	generally known as \define{Probabilistic Relational Models}.
	These representations resemble universal quantification in first-order logic
	such as $\forall x \in People \, P(happy(x) | sunny)$ for $People$ a discrete set $\{Ann, Bob, \dots\}$.
	In logic terminology, this employs \define{uninterpreted functions},
	since the random relations are Boolean functions
	without a fixed interpretation
	(for example, the value of $happy(Ann)$ is not fixed in advance).
	\item \define{Probabilistic Inference Modulo Theories (PIMT)}
	\citep{desalvobraz16probabilistic}
	exploits \define{algebraic structure}
	by explicitly representing and manipulating function definitions of random variable
	dependences in the form of algebraic expressions.
	In logic terminology, this employs \define{interpreted functions},
	because functions like equality ($=$) and arithmetic functions ($+$, $\times$ etc),
	among others, have \emph{fixed} interpretations: $3 = 3$ is always $\true$,
	and $1 + 3$ is always $4$.
\end{itemize}
In this paper, we present the \define{first lifted probabilistic algorithm on languages with interpreted functions}.
This is done by \define{unifying these two lines of research} by incorporating
uninterpreted functions and an important lifted probabilistic inference operation,
\define{Inversion}, into PIMT.
(Another major lifted operation, Counting \citep{desalvobraz07thesis, milch08lifted}, is left for future work.)
We call this fusion of lifted inference and PIMT \define{Lifted Probabilistic Inference Modulo Theories (LPIMT)}. 

Moreover, we achieve this unification \emph{modularly}, by using PIMT's ability
to be extended by solvers for specific theories,
and by encapsulating lifted inference into it as an \emph{uninterpreted functions theory solver}.
This means the algorithm can apply lifted
inference even to unanticipated theories,
because solvers for different theories are orthogonally applied by the general,
theory-agnostic level of PIMT.

Casting lifted inference in the PIMT framework also makes it simpler and more powerful than previous lifted inference methods.
Besides the above mentioned advantages, it also uses standard mathematical notation (e.g.,
not making a hard distinction between ``logical'' and ``random'' variables);
does not separate constraints from potential function definitions;
and accepts random function applications even as arguments to other functions (nested applications).

\section{BACKGROUND}

\subsection{Graphical Models and Variable Elimination}

\define{Graphical models} are a standard framework for reasoning with uncertainty.
The most common types are Bayesian networks
and Markov networks.
In both cases, a \define{joint probability distribution}
for each assignment tuple \bx to $N$ random variables
is defined as a normalized product of non-negative real functions $\{\phi_i\}_{i\in 1..M}$,
where $1..M$ is short for $\{1,\dots,M\}$,
each of them applied to a subtuple $\bx_i$ of \bx:\footnote{For simplicity, we use the same symbols for both random variables and their values, but the meaning should be clear.}
\begin{align*}
P(\bx) = \frac{1}{Z}\prod_i \phi_i(\bx_i),
\end{align*}
where $Z$ is a normalization constant equal to $\sum_{\bx} \prod_i \phi_i(\bx_i)$.
Functions $\phi_i$ are called \define{factors}
and map each assignment on their arguments to a \define{potential},
a non-negative real number that represents
how likely the assignment $\bx_i$ is.
This representation is called \define{factorized} due to its
breaking the joint probability into this product.
In Bayesian networks, $M = N$ and factors are conditional probabilities
$P(x_i | Pa_i)$, for each random variable $x_i$ in \bx, where $Pa_i$
are its parents in a directed acyclic graph.

The \define{marginal probability (MAR)} problem consists of computing $P(\bq) = \sum_{\br} P(\bx)$, where $\bq$ is a subtuple of \bx
containing queried variables, and $\br$ are all the remaining variables in $\bx$.
It can be shown that $P(\bq)=\frac{1}{Z_{\bq}} \sum_\br \prod_i \phi_i(\bx_i)$
for $Z_\bq$ a normalization constant over $\bq$.
Therefore, the problem can be easily reduced to computing a \define{summation over products of factors}, which the rest of the paper focuses on.

The \define{belief update (BEL)} problem consists of computing $P(\bq|\be)$,
the posterior probability on variables $\bq$ given assignment $\be$ to evidence variables (an observation).
BEL can be easily reduced to two instances of MAR, since Bayes' theorem establishes that $P(\bq|\be) = P(\bq, \be)/P(\be)$, where $P(\bq, \be)$ is the joint value of $\bq$ and $\be$.



Computing $\sum_{\br} \prod_i \phi_i(\bx_i)$ is crucial
for inference with graphical models, but solving it naively has
exponential cost in the size of $\br$.
\define{Variable Elimination (VE)} is an algorithm that takes advantage
of the factorized representation to more efficiently compute this sum. 
For example, consider the following summation given
a factorization on the variables ($h$)appy, ($w$)eekday, ($t$)emperature, ($m$)onth with range sizes
of $2$ (true or false), 7 (weekdays), $3$ ($hot$, $mild$, $cold$), and $12$ (months) respectively:
\begin{align*}
\sum_{h, w, t, m} P(h | w, t) P(w) P(t|m) P(m).
\end{align*}
A naive computation iterates over $504$ assignments to $(h, w, t, m)$.
However, the joint probability factorization
enables us to manipulate the expression
and compute it with iterations over fewer assignments:
\begin{align*}
& \sum_{h, w, t, m} P(h | w, t) P(w) P(t|m) P(m)
\\& = \sum_{h, w, t} P(h | w, t) P(w) \sum_m P(t|m) P(m).
\end{align*}
Now, we \define{sum} $m$ \define{out} (or \define{eliminate} it), 
obtaining a new factor $\phi$, defined on $t$ alone,
that replaces the last summation.
This requires going over each value $t$ and computing $\phi(t) = \sum_m P(t|m)P(m)$
(thus iterating over the values of $m$).
Therefore, computing $\phi$ takes iterating over $36$ assignments to $t,m$ and we
are left with the following new summation of a product of functions:
\begin{align*}
& \sum_{h, w, t} P(h | w, t) P(w) \phi(t)
\\& = \sum_w P(w) \sum_{h, t} P(h | w, t) \phi(t)
\\& = \sum_w P(w) \phi'(w) \text{ (after $42$ assignments to $w, h, t$)}
\\& = \phi'' \text{ (after $7$ assignments to $w$)}.
\end{align*}
Variable Elimination therefore decreases the number of required iterations
to $85$, a six-fold reduction, by exploiting the fact that
not all variables share factors with all variables
(for example, $m$ only shares a factor with $t$).

VE applies not only to summations of products, but to any commutative associative semiring
\citep{bistarelli97semiring}:
maximization of products, disjunctions of conjunctions, conjunctions of disjunctions, and so on.
We use $\oplus$ and $\otimes$ for the \define{additive and multiplicative operators} of the semiring.
We call \define{quantifiers} the intensional versions of operators:
$\forall$ for $\wedge$, $\exists$ for $\vee$, $\sum$ and $\int$ for $+$ (for discrete and continuous indices respectively), $\prod$ for $\times$,
Max for max, and so on.
Quantifiers corresponding to $\oplus$ and $\otimes$ are denoted  $\bigoplus$ and $\bigotimes$, respectively. We use $\bigodot$ for denoting any type of quantifier. VE works by selecting the next variable $v$ in \br to be eliminated
and computing a new factor
\[\phi(\bu) = \bigoplus_v \bigotimes_{i \in \bff(v)} \phi_i(\dots,v,\dots),\]
where $\bff(v)$ is the set of indices of factors of $v$ and $\bu$
are the variables sharing factors with $v$.
It then includes $\phi$ in the product of factors and proceeds until all variables in $\br$ are eliminated.

\subsection{Probabilistic Inference Modulo Theories }

Typically, graphical model inference algorithms assume that
factors can only be accessed as opaque lookup tables,
that is, by providing an assignment to its random variables.
This requires Variable Elimination to iterate over all assignments
to all random variables in a summation,
as seen in the example for computing $\phi(t)$ 
with Variable Elimination.

However, often the \emph{definitions}
of factors are available in the form of \define{symbolic algebraic expressions}
that use operators from specific theories.
For example, the conditional probability of temperature given month,
and the prior probability of month in our
original example could be represented by\footnote{Note that $P(t|m)$ sums up to $1$ for each $m$ because $hot$ and $mild$ both have the $\mathtt{else}$ case probability mass.}
\begin{align*}
& P(t|m) = \iftt m \leq 3
\\& \qquad \qquad \qquad \thentt \iftt t = cold \thentt 0.8 \elsett 0.1
\\& \qquad \qquad \qquad \elsett \iftt t = cold \thentt 0.4 \elsett 0.3
\\&P(m) = 1/12.
\end{align*}

While this form does not preclude regular VE to access it as a lookup table
(by computing its result given an assignment to $t$ and $m$),
the symbolic expression is itself available as a data structure.
This makes the structure of the factor evident;
for example, specific values of $m$
do not matter, but only whether $m \leq 3$.

\define{Probabilistic Inference Modulo Theories (PIMT)},
exploits this available symbolic information
for much faster inference.
It does so by using \define{Symbolic Evaluation Modulo Theories (SEMT)},
which is
a method for simplifying (and, when possible, completely evaluating)
symbolic algebraic expressions, including
eliminating quantifiers such as $\sum$, $\prod$, $\exists$, and $\forall$.
SEMT applies to general expressions, and not only
probabilistic reasoning-related ones.
PIMT is simply the use of SEMT applied to expressions
that happen to be (unnormalized) marginal probabilities,
so we mostly refer to SEMT from now on,
keeping in mind that it includes PIMT as a special case.

SEMT and PIMT are similar to
\define{Satisfiability Modulo Theories (SMT)} solvers \citep{barrett09satisfiability,moura07tutorial},
which also take modular theory-specific solvers
to solve multi-theory problems.
However, they generalize SMT in two important ways:
they are \define{quantifier-parametric}, that is,
they eliminate multiple types of quantifiers and deal with expressions of any type, including Boolean, numeric and categorical, 
as opposed to SMT only solving existential quantification
on Boolean-valued formulas;
and they are \define{symbolic}, that is, can process free variables
and return results expressed in terms of them,
as opposed to assuming all variables to be existentially quantified as in SMT.

SEMT is a variant of SGDPLL($T$) \citep{desalvobraz16probabilistic}
slightly generalized to deal with any expression, and not
just sequences of nested quantifiers, and to make use of \emph{contexts}, defined below.
We give an overview of SEMT here but refer to the SGDPLL($T$) reference for details.

SEMT receives a pair $(E, C)$ as input.
$E$ is the expression being evaluated, and $C$ the \define{context}.
Let \bx be the free variables in $E$.
Context $C$ is a formula on \bx
that indicates that \emph{only} the assignments on \bx 
satisfying $C$ need be considered.
For example, the expression
$\ifte{x \neq 1}{2}{4}$, under context $x = 2 \vee x = 3$, can be safely evaluated
to $4$ because $x \neq 1$ is never true under that context.

SEMT traverses the expression top-down and evaluates each sub-expression
according to a set of \define{rewriting rules} (presented below)
until no rule applies to any sub-expressions.
It also keeps track of the context holding for sub-expression,
depending on the expressions above it.
For example, in $\ifte{x = 1}{\sum_{i \in 1..10 : i \neq 3} \gamma}{\varphi}$, $\gamma$ is evaluated under context $x = 1 \wedge i \neq 3$.

Let $E[\alpha/\beta]$ be the \define{substitution} of all occurrences of expression $\alpha$ in expression $E$ by $\beta$.
The SEMT rewriting rules are the following (an example is provided afterwards):
\begin{itemize}[leftmargin=0.15in]
	\item \define{simplification}: simplifiers for each theory
	are applied when possible:
	examples are $1 + 2 \rightarrow 3$, $x = x \rightarrow true$, $\ifte{true}{1}{2} \rightarrow 1$, $0 \times x \rightarrow 0$ and so on.
	These always decrease the size of the expression.
	\item \define{literal determination}: if the expression is a literal $L$
	under a context $C$ and $\forall V : C \Rightarrow L$,
	where $V$ are the free variables in $C$ and $L$,
	rewrite the expression to \true;
	if $\forall V : C \Rightarrow \neg L$, rewrite it to \false;
	the tautology may be decided by SEMT itself or an SMT solver;
	\item \define{factoring out}: if $\phi_1$ does not involve index $i$,
	$\bigodot_{i \in D:C} (\phi_1 \odot \phi_2)
	\rightarrow \phi_1 \odot \bigodot_{i \in D:C} \phi_2$.
	\item \define{if-splitting}: if the expression is $\phi$ and contains
	a literal $L$ undefined by the context $C$,
	rewrite it to $\ifte{L}{\phi[L/\true]}{\phi[L/\false]}$.
	\item \define{quantifier-splitting}: rewrite expressions
	of the form $\bigodot_{i \in D:C} \phi$, where $i$ is a variable of type $D$, $F$ is a conjunction of literals in the theory for type $D$, and $\phi$ is expression containing a literal $L$ containing $i$,
	to a new expression containing two quantified expressions,
	each with one less literal in its body: $\Bigl(\bigodot_{i:F\wedge L} \phi[L/\true]\Bigr) \odot \Bigl(\bigodot_{i:F\wedge \neg L} \phi[L/\false]\Bigr)$.
	\item \define{theory-specific quantifier elimination}: if $\phi$ does not contain any literals, solve $\bigodot_{i \in D:F} \phi$ with a provided, modular \emph{theory-specific} solver for the type of $i$.
\end{itemize}
SEMT always returns a quantifier-free expression since the theory-specific quantifier elimination is eventually invoked for each quantifier. For an expression representing a marginal over a product of factors, it reproduces Variable Elimination by summing out one variable at a time,
with the advantage of exploiting factors represented
symbolically.\footnote{The variable order used by VE is encoded in the order of quantifiers,
and efficient ordering can be introduced by a rule that re-orders sequences of quantifiers.}

Consider the computation of $\phi(t) = \sum_m P(t|m)P(m)$ below.
Regular VE requires iterating over $36$ assignments to $t,m$,
but SEMT only needs to consider the $4$ truth assignments for literals $t = cold$ and $m \leq 3$:
\begin{align*}
& \sum_m (\iftt m \leq 3 \thentt \iftt t = cold \thentt 0.8 \elsett 0.1
\\& \qquad \qquad \qquad \elsett \iftt t = cold \thentt 0.4 \elsett 0.3)/12
\\& \rightarrow \text{(by if-splitting on $t = cold$ and simplification)}
\\& \;\quad \iftt t = cold
\\& \quad \qquad \thentt \sum_m (\iftt m \leq 3 \thentt 0.8 \elsett 0.4)/12
\\& \quad \qquad \elsett \sum_m (\iftt m \leq 3 \thentt 0.1 \elsett 0.3)/12
\\& \rightarrow \text{(by quantifier-splitting on $m \leq 3$ in first summation)}
\\& \;\quad \iftt t = cold
\\& \quad \qquad \thentt 
               \sum_{m:m \leq 3} \;\, (\iftt \true \thentt 0.8 \elsett 0.4)/12
\\& \quad \qquad \quad + \sum_{m:\neg(m \leq 3)} (\iftt \false \thentt 0.8 \elsett 0.4)/12
\\& \quad \qquad \elsett \sum_m (\iftt m \leq 3 \thentt 0.1 \elsett 0.3)/12
\end{align*}
\begin{align*}
& \rightarrow \text{(by simplification and quantifier-splitting on $m \leq 3$)}
\\& \;\quad \iftt t = cold \thentt 0.8/12 \times 3 + 0.4/12 \times 9
\\& \qquad \qquad \qquad \elsett 0.1/12 \times 3 + 0.3/12 \times 9
\\& \rightarrow \iftt t = cold \thentt 0.5 \elsett 0.25,
\end{align*}
which represents the resulting new factor $\phi(t)$ (which happens to be $P(t)$) to be used in the next steps of VE.
Note that the expressions above are \emph{not} just a mathematical argument, but the actual data structures manipulated by SEMT.

SEMT (and consequently PIMT) is \define{modulo theories} because symbolic evaluation is \define{theory-agnostic},
working for any theory $T$ given an encapsulated \define{theory solver} that provides quantifier elimination solvers and simplification rules for $T$.
Theory solvers must be provided only for the simpler expressions
of the type $\bigodot_{i\in D:F} \phi$ where $\phi$ is literal-free and $F$ 
is a conjunction of literals in the theory. \citet{desalvobraz16probabilistic} details a solver for summation over polynomials, with literals in difference arithmetic on bounded integers.
A similar solver for integrals over polynomials
and linear real arithmetic literals has also been defined since.
An example of SEMT on multiple theories with only $4$ cases is:
\begin{align*}
& \sum_{i\in 1..7} \int_{x \in [0;10]} \hspace{-0.5cm}\ifte{i \geq 3}{\ifte{x < 5}{x^2}{0}}{i}
\\&\rightarrow \sum_{i\in 1..7} \iftt i \geq 3
 \thentt \int_{x \in [0;10]} \ifte{x < 5}{x^2}{0}
\\& \qquad \qquad \qquad \quad \; \elsett \int_{x \in [0;10]} i
\\&\rightarrow \sum_{i\in 1..7} \iftt i \geq 3
\thentt  \int_{x \in [0;10]:x < 5} x^2  + 
      \int_{x \in [0;10]:x \geq 5} 0 
\\& \qquad \qquad \qquad \quad \elsett 10i
\\&\rightarrow \sum_{i\in 1..7} \iftt i \geq 3 \thentt 125/3 \elsett 10i
\\&\rightarrow \bigl( \sum_{i\in 1..7:i \geq 3} 125/3 \bigr) + \bigl( \sum_{i\in 1..7:i < 3} 10i \bigr)
\\&\rightarrow 625/3 + 30 \rightarrow 715/3.
\end{align*}

\subsection{Relational Graphical Models}

\define{Relational graphical models} (RGM) specify probability distributions
over families of random variables that can be seen as relations.
There are many RGM languages in the literature;
they vary significantly in syntax and superficial features,
but are typically equivalent and convertible to one another.
Markov Logic Networks (MLN) \citep{richardson04markov} is a well-known one.
It consists of a set of weight formulas, an example of which is
\vspace{-0.15cm}
\begin{align*}
& 2.5 : Smoker(Bob) \\
& 1.4: Smoker(X) \wedge Friends(X,Y) \Rightarrow Smoker(Y)
\end{align*}
for \define{logical} variables $X$ and $Y$ ranging over a finite set $D=\{Ann, Bob, Charlie, \dots\}$.

The random variables in this MLN are the \define{groundings},
or \define{instantiations}, of 
the relations in it for every assignment to the logical variables:
$Sm(Ann)$, $Sm(Bob)$, $\dots$, $Fr(Ann, Ann)$, $Fr(Ann, Bob)$, $\dots$
(we abbreviate $Smoker$ and $Friends$ from now on).
A formula with weight $w$ defines a factor for each of its instantiations (one for each assignments to its logical variables).
The potential assigned by such factors is
$e^w$ if the formula is true, and $1$ otherwise.
Therefore, some of the factors of this MLN are:
$\phi_1(Sm(Bob))$, $\phi_2(Sm(Ann), Fr(Ann, Bob), Sm(Bob))$,
$\phi_2(Sm(Ann), Fr(Ann, Charlie), Sm(Charlie))$,
and $\phi_2(Sm(Bob), Fr(Bob, Ann), Sm(Ann))$,
where $\phi_1$ and $\phi_2$ are potential functions
applied to all factors instantiated from the first and second formulas respectively:
\begin{align*}
& \phi_1(s) = \ifte{s}{e^{2.5}}{1} \\
& \phi_2(s_1, f, s_2) = \iftt s_1 \wedge f \Rightarrow s_2 \thentt e^{1.4} \elsett 1.
\end{align*}
Because the number of instantiations can be huge,
performing inference on RGMs
by simply instantiating them as regular graphical models is often infeasible.
The next section describes \define{Inversion}, one of the operations
used in \define{lifted probabilistic inference}, which exponentially
improves efficiency in RGMs in many cases.

\section{INVERSION AND INVERSION MODULO THEORIES}

This section presents this paper's main contributions.
We present a new formulation for RGMs
and Inversion
that is more algebraically standard, and then
use this formulation to generalize Inversion to
Inversion Modulo Theories.

\subsection{RGMS with Function-Valued Random Variables}
\label{sec:relational-graphical-models-with-function-valued-random-variables}

While the RGM literature considers
$Sm(Ann)$, $Sm(Bob),\dots$ as individual (Boolean) random variables,
it is conceptually simpler, and more mathematically standard,
to talk about $Sm$ as a single random variable that happens
to be \emph{function-valued}.
From this point of view, this MLN has only two random variables:
$Sm$ and $Fr$,
and defines
the following joint probability distribution:
\begin{align*}
& P(Sm, Fr) = \frac{1}{Z} \times \phi_1(Sm(Bob)) \\
& \qquad \times \prod_{X\in D} \prod_{Y\in D} \phi_2(Sm(X), Fr(X,Y), Sm(Y))
\end{align*}
where $Z$ is defined as usual, as well as marginalization:
\begin{align*}
P(Sm) = \sum_{Fr \in D \times D \rightarrow Boolean} \hspace{-0.8cm}P(Sm, Fr).
\end{align*}
Note that, given the form of $P(Sm,Fr)$,
marginal probabilities in RGMs are \define{sums of products of factors},
including \emph{intensional} products using $\prod$.
This fact is heavily exploited by Inversion.

To compute the marginal of a specific function application, say, $Sm(Ann)$,
we need to marginalize over all \emph{remaining} random variables,
that is, $Fr$ and all $Sm(x)$ for $x \in D \setminus \{Ann\}$.
This requires splitting the variable $Sm$ into two distinct function-valued random variables: $Sm': \{Ann\} \rightarrow Boolean$ and $Sm'': D \setminus \{Ann\} \rightarrow Boolean$, and replace each application $Sm(\theta)$ of the original $Sm$ on argument expression $\theta$ by \newline
$\ifte{\theta = Ann}{Sm'(Ann)}{Sm''(\theta)}$.
Then we can compute
\begin{align*}
& P(Sm') = \sum_{Sm'': D \setminus \{Ann\} \rightarrow Boolean} \sum_{Fr}  P(Sm', Sm'', Fr).
\end{align*}
The above shows that solving RGMs
is simply equivalent to allowing function-valued random variables in the model.
However, summations over functions
are expensive due to their high number of possible values ($2^{|D|^2}$ for $Fr$, for instance).
The next section describes a method that exponentially decreases this cost in some cases.
\subsection{Inversion on Function-Valued Variables}
\label{sec:inversion}
	
\define{Lifted probabilistic inference} algorithms seek to 
exploit the structure of random functions for greater efficiency.
It includes a few operations,
but in this paper we consider only one: \define{Inversion}
(also called Lifted Decomposition).

Inversion uses the fact
that summations indexed by functions of products of factors
may under certain conditions be transformed into exponentially
cheaper summations over ``slices'' of the original function. 
Its name comes from the fact that a summation of products
becomes a cheaper product of a summation ($\sum\prod \rightarrow \prod\sum$),
thus ``inverting'' the quantifiers.
Consider an example in which the body $\phi$ of the sum-product
depends on a single application of a function $f$ ranging
over the set of functions $1..10 \rightarrow 1..5$:
\begin{align*}
& \sum_{f \in (1..10 \rightarrow 1..5)} \; \prod_{x \in 1..10} \phi(f(x)) = \prod_{x \in 1..10} \; \sum_{f \in (\{x\} \rightarrow 1..5)} \hspace{-0.4cm}\phi(f(x)).
\end{align*}
Note that, while $f$ ranges over $(1..10 \rightarrow 1..5)$ on the left-hand side,  and thus over $5^{10}$ possible values,
the domain of $f$ on the right-hand side is a singleton
set, because $x$ is bound by the now outer $\prod_x$.
This reduces the number of possible values of $f$ to only $5$,
making iterating over it exponentially cheaper.

The equality holds because $x$ ``slices'' $f$ into
independent portions:
\begin{align*}
& \sum_{f \in (1..10 \rightarrow 1..5)} \quad \prod_{x \in 1..10} \phi(f(x))
\\& = \hspace{-.3cm} 
\sum_{f_1 \in \{1\} \rightarrow 1..5}
\dots  \hspace{-.2cm}
\sum_{f_{10} \in \{10\} \rightarrow 1..5} 
\phi(f_1(1)) \times ... \times \phi(f_{10}(10))
\\& = \hspace{-.3cm}
\sum_{f_1 \in \{1\} \rightarrow 1..5} \phi(f_1(1))
\;\; \dots  \hspace{-.15cm}
\sum_{f_{10} \in \{10\} \rightarrow 1..5} \phi(f_{10}(10))\quad\text{(*)}
\end{align*}
\begin{align*}
& = 
\Bigl( \sum_{f_1 \in \{1\} \rightarrow 1..5} \phi(f_1(1)) \Bigr)
\dots
\Bigl( \sum_{f_{10} \in \{10\} \rightarrow 1..5} \phi(f_{10}(10)) \Bigr)
\\& = 
\Bigl( \sum_{f \in \{1\} \rightarrow 1..5} \phi(f(1)) \Bigr)
\dots
\Bigl( \sum_{f \in \{10\} \rightarrow 1..5} \phi(f(10)) \Bigr)
\end{align*}
\begin{align*}
& = \prod_{x \in 1..10} \; \sum_{f \in (\{x\} \rightarrow 1..5)} \phi(f(x)).
\end{align*}
Once transformed in this way, we proceed as follows:
\begin{align*}
&   \prod_{x \in 1..10} \; \phi(1) + \phi(2) + \phi(3) + \phi(4) + \phi(5) 
\\& = \prod_{x \in 1..10} \; \phi' \quad = \quad (\phi')^{10} \quad = \quad \phi''
\end{align*}
by using the fact that constant $\phi'$ does not depend on $x$.

The above transformation is valid because $\prod_x \phi(f(x))$ contains $10$ factors,
each of them only involving the application of $f$ to a \emph{single} value of $x$.
This means they can be factored out of the summations indexed by other applications of the function,
resulting in smaller and equivalent summations that are computed only once and then exponentiated.
Similar transformations may be applied even if there are products over more than
one variable:
\[
\sum_{f \in A_1 \times A_2 \rightarrow B} \prod_x \prod_y \phi(f(x,y))
= \prod_x \prod_y \hspace{-.8cm} \sum_{\qquad f \in \{(x,y)\} \rightarrow B} \hspace{-.9cm} \phi(f(x,y))
\]
by an analogous argument.

However, the transformation is not always valid:
\begin{align*}
& \sum_{f \in A_1 \times A_2 \rightarrow B} \prod_x \prod_y \phi(f(x,y),f(y,x))
\\& \qquad \neq \prod_x \prod_y \sum_{f \in \{(x,y)\} \rightarrow B} \phi(f(x,y),f(y,x)).
\end{align*}
To see this, consider a pair $(a,b) \in A_1 \times A_2$.
The summation $\sum_{f \in \{(a,b)\} \rightarrow B} \phi(f(a,b),f(b,a))$
depends on $f(a,b)$ and $f(b,a)$, which both occur in
$\sum_{f \in \{(b,a)\} \rightarrow B} \phi(f(b,a),f(a,b))$,
and this shared dependence prevents factoring as in step (*) above.

This does not mean that having more than one application of $f$ in $\phi$
admits no Inversion, but it may restrict the inversion to just \emph{some}
of the products:
\begin{align}
\label{exa:partial-inversion}
&   \sum_{f \in ((A_1 \times A_2 \times A_3) \rightarrow B)} \prod_x \prod_y \phi(f(x, y, x), f(x, 3, x)) 
\\& = \prod_x \sum_{f \in ((\{x\} \times A_2 \times \{x\}) \rightarrow B)} \prod_y \phi(f(x, y, x), f(x, 3, x)), \nonumber
\end{align}
which does not decrease the function's domain size to $1$ but still exponentially
decreases the evaluation cost.


All these cases are covered by the following theorem:
\begin{theorem}[Inversion]
	\label{the:inversion}
	Let $E$ be an expression in which \emph{all} applications of a function $f$
	have their first $k$ arguments (without loss of generality because they can be permutated)
	equal to $x_{i_1},\dots,x_{i_k}$ 
	for $i$ a $k$-tuple of indices in $\{1,\dots,m\}$.
	If operator $\oplus$ distributes over operator $\otimes$,
	\begin{align*}
	& \bigoplus_{f \in (A_1 \times \dots \times A_n) \rightarrow B} \bigotimes_{x_1} \dots \bigotimes_{x_m} E
	\\& \qquad =
	\bigotimes_{x_1} \dots \bigotimes_{x_m} \bigoplus_{f \in (\{(x_{i_1}, \dots, x_{i_k})\} \times A_{k+1} \times \dots \times A_n) \rightarrow B} E,
	\end{align*}
	which is exponentially cheaper to evaluate.
\end{theorem}
Note that $E$ may contain quantifiers itself; this was the case in Equation (\ref{exa:partial-inversion}). The theorem's proof mirrors the operations
shown in the examples above.

Crucially, Theorem \ref{the:inversion} uses a \emph{syntactic} check
that relies on the language allowing only simple terms
(variables and constants) as arguments to functions.
It therefore does not apply to the important extensions in which more complex terms,
using theory-specific operators such as arithmetic or inequalities,
or even operators from unanticipated theories,
are used as function arguments.
We will see how this can be done with a \emph{semantic} check in the next section.

\subsection{Inversion Modulo Theories}

We now present this paper's main contribution:
\define{Inversion Modulo Theories},
which is a version of the lifted probabilistic inference operation
Inversion in the presence of factors
defined through symbolic algebraic expressions involving multiple theories.
Like regular Inversion, Inversion Modulo Theories does not cover every possible problem,
in which case one must fall back to currently used methods like grounding or sampling.
Future work includes generalizing other lifted inference methods like Counting \citep{desalvobraz07thesis,milch08lifted}
to also be modulo theories, thus decreasing the need for the fallback methods.

Previous work on Inversion assumed RGMs expressed in
a simple language
in which the only allowed arguments in function applications
are variable or constant symbols.
This enables
the Inversion condition to consist of a \emph{syntactic} check.
However, once we allow function arguments to be arbitrary terms from any of the available theories, these syntactic tests no longer apply,
since the arguments semantics depends on the particular theory
and the syntactic check does not take semantics into account.
The new Inversion test must apply to arguments from any theory and,
in fact, to arguments in even \emph{new, as of yet unanticipated} theories.
This requires this new test to be theory-agnostic,
and defined in such a way that theory solvers may be transparently employed.

We start by defining $oc_f[E]$, an expression representing
the portion of the domain of a given function $f$ involved in a given expression $E$.

\begin{definition}
\label{def:oc}
The \define{set of argument tuples for $f : A \rightarrow B$ occurring in an expression $E$} is denoted $oc_f[E]$ and inductively defined as follows:

\begin{itemize}
	\item if $E$ does not contain $f$, $oc_f[E]$ is $\emptyset$;

	\item if $E$ is $f(t)$ for $t$ a tuple, $oc_f[E]$ is $\{t\}$;

	\item if $E$ is $f$, $oc_f[E]$ is $A$;

	\item if $E$ is $\ifte{C}{E_1}{E_2}$ and $C$ does not contain $f$,
	then $oc_f[E]$ is $\ifte{C}{oc_f[E_1]}{oc_f[E_2]}$;
	otherwise, $oc_f[E]$ is $oc_f[C] \cup oc_f[E_1] \cup oc_f[E_2]$;

	\item if $E$ is $g(t_1,\dots,t_k)$ for $g$ a function symbol distinct from $f$,
	or $E$ is $\{t_1,\dots,t_k\}$,
	$oc_f[E]$ is \\
	$oc_f[t_1] \cup \dots \cup oc_f[t_k]$;

	\item if $E$ is $Q_{x \in T : C} E'$ for $Q$ an arbitrary quantifier,
	\begin{itemize}
		\item if $C$ does not contain $f$, \\
		then $oc_f[E]$ is $oc_f[T] \cup \bigcup_{x \in T : C} oc_f[E']$;
		\item otherwise, $oc_f[E]$ is \\
		$oc_f[T] \cup \bigcup_{x \in T} (oc_f[C] \cup oc_f[E'])$.
	\end{itemize}
\end{itemize}
\end{definition}

For example,
\begin{align*}
& oc_f\bigl[f(x,y) + \sum_{z \in 1..10 : z != 3} f(x,z) \prod_{w \in 1..10} f(w,z)\bigr] \\
& = \{(x,y)\} \cup \bigcup_{z \in 1..10 : z != 3} \bigl(  \{(x,z)\} \cup \bigcup_{w \in 1..10} \{(w,z) \} \bigr),
\end{align*}
in which $x$ and $y$ are free variables.

As a further preliminary step,
we add simplification rules for tuples, and for testing whether sets are empty, to be used by SEMT.
This will be useful later in manipulating $oc_f[E]$ expressions.
\begin{restatable}[Tuple and Empty Set Simplifiers]{theorem}{firsttupleandemptysetsimplifiers}
	\label{the:tuple-and-empty-set-simplifiers}
	The following \define{tuple and empty set simplifiers}
\begin{enumerate}
	\item $(r_1,\dots,r_n)=(s_1,\dots,s_n) \rightarrow$ \\
	$r_1 =s_1\wedge \dots \wedge r_m = s_m$ (or its negation for $\neq$).
    \item $t \in \{t_1,\dots,t_n\} \rightarrow (t = t_1) \vee t \in \{t_2,\dots,t_n\}$.
	\item $t \in \bigcup_{i \in D : C} \phi \rightarrow \exists i \in D : (C \wedge t \in \phi)$.
	\item $\bigl( \bigcup_{i \in D : C} \phi \bigr) = \emptyset  \rightarrow  
	\forall i \in D : (\neg C \vee \phi = \emptyset)$.
    \item $S \cap \emptyset = \emptyset \rightarrow \true$.
    \item $S \cup \emptyset = \emptyset \rightarrow S = \emptyset$.
	\item $S \cap \{t_1,\dots,t_n\} = \emptyset \rightarrow$ \\
	$(t_1 \notin S) \wedge (S \cap \{t_2,\dots,t_n\} = \emptyset)$.
    \item $\bigl(\bigcup_{i \in D : C} \phi\bigr) \cap \bigl(\bigcup_{i' \in D' : C'} \phi'\bigr) = \emptyset \rightarrow$\\
	$\forall i \in D : C \Rightarrow \forall i' \in D' : C' 
	\Rightarrow (\phi \cap \phi' = \emptyset)$.
    \item $(S_1 \cup S_2) \cap S_3 = \emptyset \rightarrow (S_1 \cap S_3) \cup (S_2 \cap S_3) = \emptyset$.
    \item $S_1 \cup S_2 = \emptyset \rightarrow S_1 = \emptyset \wedge S_2 = \emptyset$.
	\item $\{t_1,\dots,t_n\} = \emptyset \rightarrow \false$ if $n>0$, \true otherwise,
\end{enumerate}
\vspace{-.2cm}
when included in SEMT, rewrite $oc_f[E]=\emptyset$ expressions to equivalent
formulas free of tuple and set expressions.
\end{restatable}
\vspace{-.2cm}
The proof is presented in Supplementary Materials.
It is based on gradual distribution of $\cap$ over $\cup$,
and conversion of intersections to comparisons between set elements:
\begin{align*}
& \bigcup_{x\in 1..5} \Bigl( \{(x, y)\} \cup \bigcup_{z \in 3..5} \{(z, 3)\} \Bigr) \; \cap \; \bigcup_{w \in 1..10} \{(1,w)\} = \emptyset
\\& \rightarrow  \text{(rule 8)}
\\& \forall x\in 1..5 : \forall w \in 1..10 : 
\\& \quad \Bigl( \{(x, y)\} \cup \bigcup_{z \in 3..5} \{(z, 3)\} \Bigr) \; \cap \; \{(1,w)\} = \emptyset
\\& \rightarrow  \text{(rule 9 distributes $\cap$ over $\cup$)}
\\& \forall x\in 1..5 : \forall w \in 1..10 : 
\\& \quad \{(x, y)\} \cap \{(1,w)\} \; \cup \;
           \Bigl( \bigcup_{z \in 3..5} \{(z, 3)\} \Bigr) \cap \{(1,w)\} = \emptyset
\\& \rightarrow  \text{(rule 10)}
\\& \forall x\in 1..5 : \forall w \in 1..10 : 
\\& \{(x, y)\} \cap \{(1,w)\} = \emptyset \wedge
\Bigl( \bigcup_{z \in 3..5} \{(z, 3)\} \Bigr) \cap \{(1,w)\} = \emptyset
\\& \rightarrow  \text{(rules 7, 2, and 8)}
\\& \forall x\in 1..5 : \forall w \in 1..10 : 
\\& \quad (x,y) \neq (1, w) \wedge \forall z \in 3..5 : \{ (z, 3)\} \cap \{(1,w)\} = \emptyset
\\& \rightarrow  \text{(rule 1 on first conjunct, rules 7 and 2 on second one)}
\\& \forall x\in 1..5 : \forall w \in 1..10 : 
\\& \quad (x \neq 1 \vee y \neq w) \wedge \forall z \in 3..5 : (z, 3) \neq (1,w)
\\& \rightarrow  \text{(rule 1 for breaking tuples)}
\\& \forall x\in 1..5 : \forall w \in 1..10 : 
\\& \quad (x \neq 1 \vee y \neq w) \wedge \forall z \in 3..5 : z \neq 1 \vee 3 \neq w
\\& \rightarrow  \text{(integer-specific solver for {\scriptsize$\bigodot$} = $\forall$ for $z$, $w$ and $x$)}
\\& \forall x\in 1..5 : \forall w \in 1..10 : 
    (x \neq 1 \vee y \neq w) \wedge 3 \neq w
\\& \rightarrow \forall x\in 1..5 : \false \quad \rightarrow \quad \false.
\end{align*}
Deciding whether a set expression is equivalent to the empty set is crucial for Inversion Modulo Theories, which can now be formally stated:
\newpage
\begin{restatable}[Inversion Modulo Theories]{theorem}{firstinversion}
	\comment{We might want to remove the constraints $C_i$ from the theorem,
		by allowing the types $T_i$ to be intensional sets.
		But then we need to allow them as types and explain how to deal with such types.
		It seems worth it, as it is the most elegant and general way of doing it}
\label{the:inversion-modulo-theories}
Let $E$ be an expression and $T_i$, $C_i$ be
type and constraint, respectively, in a theory for which we have a satisfiability solver. Then,
\begin{align*}
& \bigoplus_{f \in A \rightarrow B} \quad \bigotimes_{x_1 \in T_1 : C_1} \dots \bigotimes_{x_k \in T_k : C_k} E,
\end{align*}
where $A = oc_f[\bigotimes_{x_1 \in T_1 : C_1} \dots \bigotimes_{x_k \in T_k : C_k} E]$ (this is relaxed in Section \ref{sec:arbitrary-domains-for-f}),   
is equivalent to
, and therefore can be rewritten as,
\begin{align*}
\bigotimes_{x_1 \in T_1 : C_1} \dots \bigotimes_{x_k \in T_k:C_k} \quad \bigoplus_{f \; \in \; oc_f[E] \rightarrow B} E,
\end{align*}
if $\bigotimes$ distributes over $\bigoplus$ and
\begin{align*}
&
\forall x'_1 \in T_1 \dots \forall x'_k \in T_k \quad
\forall x''_1 \in T_1 \dots \forall x''_k \in T_k \\
&
\begin{pmatrix}
\quad C_1[x_1/x'_1] \wedge \dots \wedge C_k[x_k/x'_k] \\
\,\; \wedge \; C_1[x_1/x''_1] \wedge \dots \wedge C_k[x_k/x''_k] \\
\wedge \; (x'_1, \dots, x'_k) \neq (x''_1, \dots, x''_k)
\end{pmatrix} \\
& \Rightarrow
\begin{pmatrix}
oc_f[E][x_1/x'_1,\dots,x_k/x'_k] \\
\cap \\
oc_f[E][x_1/x''_1,\dots,x_k/x''_k] 
\end{pmatrix}  = \emptyset \numberthis \label{eqn:inversion-condition}
\end{align*}
(that is, the set of argument tuples in applications of $f$ in $E$
for different value assignments to $x_1,\dots,x_k$ are always disjoint
--- we refer to this as \define{Condition (\ref{eqn:inversion-condition})} from now on).
Moreover,
the rewritten expression is exponentially ($O(2^{\prod_i |\{x_i \in T_i : C_i\}|})$) cheaper to evaluate than the original.\\
\end{restatable}
\vspace{-.5cm}
The theorem (proven in Supplementary Materials) covers significantly more general cases than the original examples may suggest: $f$ may have any arity;
its arguments need not be only variables or constants,
but even other (interpreted or uninterpreted) function applications;
$E$ may contain quantifiers itself;
and quantifiers may involve constraints ($C_i$).
The following example involves all these characteristics:
\begin{example}
Let $w$ and $g$ be free variables.
\label{exa:main-example}
\begin{align*}
& \sum_{f \in 1..10 \times (1..10 \setminus \{8\}) \times \{w + 3\} \rightarrow 1..5} \\
& \qquad \qquad  \prod_{x \in (1 + g(w))..(10 + g(w))}
\;
\prod_{y \in 1..10 : y \neq 8}
\\&
\qquad \qquad \qquad \qquad \qquad \sum_{z \in 1..10}
f(x-g(w),y,w+3)z
\\& \rightarrow \text{ (Inversion on $x,y$;
	see Eqn. (\ref{eqn:condition-2-for-main-example}) for Condition (\ref{eqn:inversion-condition}))}
\\&
\prod_{x \in (1 + g(w))..(10 + g(w))}
\quad
\prod_{y \in 1..10 : y \neq 8}
\\& \sum_{f \in \{(x-g(w), y, w + 3)\} \rightarrow 1..5} \;\; \sum_{z \in 1..10} f(x-g(w),y,w+3)z
\\& \rightarrow \text{ ($f$ has a singleton domain since $x$, $y$, $g$ and $w$ are fixed}
\\& \qquad\text{in the first summation, so it behaves like a variable $v$)}
\\&\quad
\prod_{x \in (1 + g(w))..(10 + g(w))}
\quad
\prod_{y \in 1..10 : y \neq 8}
\quad
\sum_{v \in 1..5}
\quad \sum_{z \in 1..10} vz.
\end{align*}
This transformation decreases the time complexity
of the first summation
from the time to iterate over $5^{10\times 9}$ values of $f$
to $5$ values of $v$ only.
But, in fact, from now on SEMT completes the calculation
with no iteration at all by using symbolic integer-specific solvers for $\sum$ and $\prod$:
\begin{align*}
&
\prod_{x \in (1 + g(w))..(10 + g(w))}
\;
\prod_{y \in 1..10 : y \neq 8}
\quad
\sum_{v \in 1..5}
\quad \sum_{z \in 1..10} vz
\\& \rightarrow \prod_{x \in (1 + g(w))..(10 + g(w))}
\quad
\prod_{y \in 1..10 : y \neq 8}
\quad
\sum_{v \in 1..5} 55v
\\& \rightarrow
\hspace{-.5cm} 
\prod_{x \in (1 + g(w))..(10 + g(w))}
\prod_{y \in 1..10 : y \neq 8} 
\hspace{-.5cm} 55\hspace{-.1cm}\times\hspace{-.1cm}15 \rightarrow  825^{10\times9} \rightarrow 825^{90}.
\end{align*}
In this example, Condition (\ref{eqn:inversion-condition}) is
\begin{align}
& \forall x' \in (1 + g(w))..(10 + g(w))
\quad
\forall y' \in 1..10 \notag
\\&
\forall x'' \in (1 + g(w))..(10 + g(w))
\quad
\forall y'' \in 1..10 \notag
\\&\quad (x', y') \neq (x'', y'') \wedge 
y' \neq 8 \wedge y'' \neq 8 \Rightarrow \notag
\\& \qquad \quad \bigcup_{z \in 1..10} \hspace{-1mm}(x' - g(w), y', w+3) \notag
\\& \qquad \cap \bigcup_{z \in 1..10} (x'' - g(w), y'', w+3) = \emptyset \label{eqn:condition-2-for-main-example}
\end{align}
which can be solved by SEMT 
with $\bigoplus$ instantiated as $\forall$
over difference arithmetic, after including tuple and empty set simplifiers from Definition \ref{the:tuple-and-empty-set-simplifiers}).
\end{example}

\begin{example}
	Consider monitoring crop growth from satellite images to alert against famine. The growth ($g$)  of crops can be determined from their color ($c$) in the images, and depends on whether the region was in drought ($d$) 3 months previously. This can be modeled as:
\begin{align*}
& \forall m \in Months, f \in Fields : \\
& \hspace{.15cm} P( c(f, m) | g(f, m) ) = \text{if } g(f, m) > 2.3 \text{ then if } c(f, m) \dots \\
& \hspace{.15cm} P (g(f, m + 3) | d(m)) = \text{if } d(m) \text{ then if } g(f, m + 3) \dots
\end{align*}
The query $P( d | c )$ requires solving the following marginalization over ‘growth’:
\[
\sum_{g}  \prod_m \prod_f P(c(f, m + 3) | g(f, m + 3)) P(g(f, m + 3) | d(m))
\]
Inversion Modulo Theories applies (because each $(f, m)$ involves a single instance of $g(f, m + 3))$ and we obtain
\[
\prod_m \prod_f \hspace{-.15cm}  \sum_{g(f, m + 3)} \hspace{-.35cm} P(c(f, m + 3) | g(f, m + 3)) P(g(f, m + 3) | d(m))
\]
which makes the summation on growth exponentially cheaper to compute and produces
\[\prod_m \prod_f \phi(c(f, m + 3), d(m)),\]
for $\phi$ the summation result. This can then be evaluated directly, given the evidence on color for each $m$ and $f$, producing a factorized representation of the marginal on $drought(m)$, for each month. If the number of fields is 2000, and growth’s domain size is 5, this cuts the cost of exactly eliminating $growth$ from $5^{2000}$ iterations to only 5.
\end{example}

\subsection{Dealing with Arbitrary Domains for $f$}
\label{sec:arbitrary-domains-for-f}

Let $OC_f$ be $oc_f[\bigotimes_{x_1 \in T_1 : C_1} \dots \bigotimes_{x_k \in T_k : C_k} E]$.
Theorem \ref{the:inversion-modulo-theories} requires that 
$A = OC_f$,
that is, that the domain of $f$ coincide with its portion being used inside $\bigoplus_{f \in A \rightarrow B}$.
When $A \not= OC_f$, we use function splitting (Section
\ref{sec:relational-graphical-models-with-function-valued-random-variables}):
\begin{align*}
&\bigoplus_{f \in A \rightarrow B}
\hspace{-.15cm}
\Psi = 
\hspace{-.15cm}
\bigoplus_{f' \in (A \setminus OC_f) \rightarrow B}
\bigoplus_{f'' \in (OC_f \setminus A) \rightarrow B}
\bigoplus_{f \in (A \cap OC_f) \rightarrow B}
\hspace{-.2cm}
\Psi'
\end{align*} 
for $\Psi'$ obtained from $\Psi$ after replacing
each $f(\alpha)$ by
\begin{align*}
&\iftt{\alpha \in A}
\\&\quad 
    \thentt
	\ifte{\alpha \in OC_f}{f(\alpha)}{f'(\alpha)}
	\elsett{f''(\alpha)}.
\end{align*}
This technique splits the original $f$ into three different functions $f$, $f'$, $f''$,
and replaces each of its applications by the one corresponding to its arguments.
After this, the domain of $f$ coincides with $OC_f$, satisfying the corresponding
requirement in Theorem \ref{the:inversion-modulo-theories}.

\subsection{Dealing with Multiple Separate $\bigotimes$ Quantifiers}

Theorem \ref{the:inversion-modulo-theories} requires a single
nested sequence of $\bigotimes$ quantifiers inside $\bigoplus$.
However, summations on products of \emph{separate} $\bigotimes$ quantifiers can be rewritten to the required form:
\begin{align*}
&
\bigoplus_{f \in A \rightarrow B} \Bigg( \Bigl( \bigotimes_{x \in T_x:C_x} E_x \Bigr) \Bigl( \bigotimes_{y \in T_y:C_y} E_y \Bigr) \Bigg)
\\& \rightarrow
\bigoplus_{f \in A \rightarrow B} \;
\bigotimes_{x \in T_x:C_x} \;
\bigotimes_{y \in T_y:C_y} 
E_x^\frac{1}{|\{y \in T_y : C_y \}|}
E_y^\frac{1}{|\{x \in T_x : C_x \}|},
\end{align*}
where the exponents compensate for the extra multiplications of $E_x$ and $E_y$
by moving them inside $\bigotimes_y$ and $\bigotimes_x$ respectively
(this is sometimes called \define{scaling}).
Set cardinalities can be computed as a special case of summation.

In the special case in which $x$ and $y$ range
over the same values (that is, $\{x \in T_x : C_x \} = \{y \in T_y : C_y \}$,
which can be evaluated by SEMT),
the two $\bigotimes$ quantifiers can be merged into a single one (or \define{aligned}), and the original expression is rewritten instead to the cheaper
$\bigoplus_{f \in A \rightarrow B} \; \bigotimes_{x \in T_x:C_x} E_x E_y[y/x]$.
%
These transformations can be easily generalized to cases with more than two $\bigotimes$ quantifications, and nested $\bigotimes$ expressions.

\subsection{A Proof-of-concept Experiment}

Since SEMT and Inversion Modulo Theories are elaborate symbolic
algorithms, two immediate questions are whether
they can be effectively implemented and 
how they compare with simpler alternatives such as sampling.
As a proof-of-concept test, we used our implementation of SEMT
for computing Example \ref{exa:main-example}
using two alternatives: one, in which 
the sum-product is simplified by Inversion Modulo Theories,
and another in which the sum and product are computed by sampling.
As expected, Inversion Modulo Theories vastly outperforms sampling
in this example, computing the exact answer ($825^{90}$) in less than 300 ms,
whereas sampling requires 10 minutes, and 10,000 samples per quantifier,
to be within an order of magnitude of the exact answer, and 17 minutes to be within 10\% error.

\section{RELATED WORK AND CONCLUSION}

As mentioned in the introduction, LPIMT generalizes work in the lifted inference, probabilistic inference modulo theories, and satisfiability modulo theories literatures.
It is also related to Probabilistic Programs (PP) \citep{goodman12church,milch05blog}, a class of high-level representations for
probabilistic models that uses interpreted and uninterpreted functions.
The current prevalent inference method in PPs is sampling,
which is approximate and whose convergence rate depends on the size of the grounded
model, that is, on the size of the domain.
LPIMT can be applied to some fragments of PPs and is an exact inference  alternative that avoids iterating over the domain.
The closest approach to LPIMT in the PP area is Hakaru \citep{carette16simplifying,narayanan2016probabilistic}, which employs 
symbolic methods for simplification and integration of PPs,
but does not include lifted inference on random functions.
Model Counting Modulo Theories \citep{phan2015model} leverages SMT solvers to compute model counts, but
does not cover weighted model counting (and thus
probabilistic reasoning), and does not exploit factorization.
Weighted Model Counting and Integration \citep{michels15new,belle15probabilistic,michels16approximate} and Symbolic Variable Elimination \citep{sanner2012symbolic} are similar in spirit to PIMT, but not to LPIMT;
they apply to Boolean and linear real arithmetic random variables,
but not yet to random functions.
\citep{belle17weighted} focuses on weighted model counting (WMC) with function symbols on infinite domains, reducing it to standard WMC, but does not currently use lifting techniques. Group Inversion \citep{taghipour12lifted} expands
Inversion to cover some extra cases, but not interpreted functions.
Extending it to do so, along with Counting,
is the most immediate possibility for future work.

To conclude, we have defined Inversion Modulo Theories, an expansion of the lifted inference operation
Inversion for the case with interpreted functions,
and added it to Probabilistic Inference Modulo Theories framework,
thus defining the first algorithm to perform exact lifted inference in the presence
of interpreted functions.

\paragraph*{ACKNOWLEDGMENTS}
We gratefully acknowledge the support of the Defense Advanced Research Projects Agency (DARPA) Probabilistic Programming for Advanced Machine Learning Program under Air Force Research Laboratory (AFRL) prime contract no. FA8750-14-C-0005.

\newpage
\bibliography{bib}

\begin{thebibliography}{}

\bibitem[\protect\citeauthoryear{Barrett \bgroup \em et al.\egroup
  }{2009}]{barrett09satisfiability}
Barrett, C.~W., Sebastiani, R., Seshia, S.~A., and Tinelli, C.
\newblock {Satisfiability Modulo Theories}.
\newblock In Biere, A., Heule, M., van Maaren, H., and Walsh, T., editors, {\em
  {Handbook of Satisfiability}}, volume 185 of {\em {Frontiers in Artificial
  Intelligence and Applications}}, pages 825--885. IOS Press, 2009.

\bibitem[\protect\citeauthoryear{Belle \bgroup \em et al.\egroup
  }{2015}]{belle15probabilistic}
Belle, V., Passerini, A., and {Van den Broeck}, G.
\newblock Probabilistic inference in hybrid domains by weighted model
  integration.
\newblock In {\em Proceedings of 24th International Joint Conference on
  Artificial Intelligence (IJCAI)}, 2015.

\bibitem[\protect\citeauthoryear{Belle}{2017}]{belle17weighted}
Belle, V.
\newblock Weighted model countingwith function symbols.
\newblock In {\em {Proceedings of the Conference on Uncertainty in Artificial
  Intelligence (UAI)}}, 2017.

\bibitem[\protect\citeauthoryear{Bistarelli \bgroup \em et al.\egroup
  }{1997}]{bistarelli97semiring}
Bistarelli, S., Montanari, U., and Rossi, F.
\newblock Semiring-based constraint satisfaction and optimization.
\newblock {\em J. ACM}, 44(2):201--236, March 1997.

\bibitem[\protect\citeauthoryear{Boutilier \bgroup \em et al.\egroup
  }{1996}]{csi}
Boutilier, C., Friedman, N., Goldszmidt, M., and Koller, D.
\newblock Context-{S}pecific {I}ndependence in {B}ayesian {N}etworks.
\newblock In {\em Proceedings of UAI}, pages 115--123, 1996.

\bibitem[\protect\citeauthoryear{Carette and Shan}{2016}]{carette16simplifying}
Carette, J. and Shan, C.
\newblock Simplifying probabilistic programs using computer algebra.
\newblock In {\em Practical Aspects of Declarative Languages - 18th
  International Symposium, {PADL} 2016, St. Petersburg, FL, USA, January 18-19,
  2016. Proceedings}, pages 135--152, 2016.

\bibitem[\protect\citeauthoryear{de Moura \bgroup \em et al.\egroup
  }{2007}]{moura07tutorial}
de~Moura, L., Dutertre, B., and Shankar, N.
\newblock A tutorial on satisfiability modulo theories.
\newblock In {\em Computer Aided Verification, 19th International Conference,
  CAV 2007, Berlin, Germany, July 3-7, 2007, Proceedings}, volume 4590 of {\em
  Lecture Notes in Computer Science}, pages 20--36. Springer, 2007.

\bibitem[\protect\citeauthoryear{{de Salvo Braz} \bgroup \em et al.\egroup
  }{2016}]{desalvobraz16probabilistic}
{de Salvo Braz}, R., O'Reilly, C., Gogate, V., and Dechter, R.
\newblock {Probabilistic Inference Modulo Theories}.
\newblock In {\em {Proceedings of the {T}wenty-Fifth {I}nternational {J}oint
  {C}onference on {A}rtificial {I}ntelligence}}, New York, USA, 2016.

\bibitem[\protect\citeauthoryear{de Salvo~Braz}{2007}]{desalvobraz07thesis}
de~Salvo~Braz, R.
\newblock {\em Lifted First-Order Probabilistic Inference}.
\newblock PhD thesis, University of Illinois at Urbana-Champaign, 2007.

\bibitem[\protect\citeauthoryear{Goodman \bgroup \em et al.\egroup
  }{2012}]{goodman12church}
Goodman, N.~D., Mansinghka, V.~K., Roy, D.~M., Bonawitz, K., and Tarlow, D.
\newblock Church: a language for generative models.
\newblock {\em CoRR}, abs/1206.3255, 2012.

\bibitem[\protect\citeauthoryear{Kersting}{2012}]{kersting12lifted}
Kersting, K.
\newblock Lifted probabilistic inference.
\newblock In {\em European Conference on Artificial Intelligence}, 2012.

\bibitem[\protect\citeauthoryear{Michels \bgroup \em et al.\egroup
  }{2015}]{michels15new}
Michels, S., Hommersom, A., Lucas, P. J.~F., and Velikova, M.
\newblock A new probabilistic constraint logic programming language based on a
  generalised distribution semantics.
\newblock {\em Artificial Intelligence}, 228(C):1--44, November 2015.

\bibitem[\protect\citeauthoryear{Michels \bgroup \em et al.\egroup
  }{2016}]{michels16approximate}
Michels, S., Hommersom, A., and Lucas, P. J.~F.
\newblock Approximate probabilistic inference with bounded error.
\newblock In {\em {Proceedings of the {T}wenty-Fifth {I}nternational {J}oint
  {C}onference on {A}rtificial {I}ntelligence}}, New York, USA, 2016.

\bibitem[\protect\citeauthoryear{Milch \bgroup \em et al.\egroup
  }{2005}]{milch05blog}
Milch, B., Marthi, B., Russell, S., Sontag, D., Ong, D.~L., and Kolobov, A.
\newblock {BLOG}: probabilistic models with unknown objects.
\newblock In {\em IJCAI'05: Proceedings of the 19th international joint
  conference on Artificial intelligence}, pages 1352--1359, San Francisco, CA,
  USA, 2005. Morgan Kaufmann Publishers Inc.

\bibitem[\protect\citeauthoryear{Milch \bgroup \em et al.\egroup
  }{2008}]{milch08lifted}
Milch, B., Zettlemoyer, L., Kersting, K., Haimes, M., and Kaelbling, L.~P.
\newblock Lifted probabilistic inference with counting formulas.
\newblock In {\em Proceedings of the Twenty-Third AAAI Conference on Artificial
  Intelligence (AAAI-2008)}, Chicago, Illinois, USA, July 2008 2008.

\bibitem[\protect\citeauthoryear{Narayanan \bgroup \em et al.\egroup
  }{2016}]{narayanan2016probabilistic}
Narayanan, P., Carette, J., Romano, W., Shan, C.-c., and Zinkov, R.
\newblock Probabilistic inference by program transformation in hakaru (system
  description).
\newblock In Kiselyov, O. and King, A., editors, {\em Functional and Logic
  Programming: 13th International Symposium, FLOPS 2016, Kochi, Japan, March
  4-6, 2016, Proceedings}, pages 62--79. Springer International Publishing,
  Cham, 2016.

\bibitem[\protect\citeauthoryear{Pearl}{1988}]{pearl88probabilistic}
Pearl, J.
\newblock {\em Probabilistic reasoning in intelligent systems: networks of
  plausible inference}.
\newblock Morgan Kaufmann, San Mateo (Calif.), 1988.

\bibitem[\protect\citeauthoryear{Phan}{2015}]{phan2015model}
Phan, Q.-S.
\newblock {\em Model {C}ounting {M}odulo {T}heories}.
\newblock PhD thesis, Queen Mary University of London, 2015.

\bibitem[\protect\citeauthoryear{Poole}{2003}]{poole03first}
Poole, D.
\newblock First-order probabilistic inference.
\newblock In {\em Proceedings of the 18th International Joint Conference on
  Artificial Intelligence}, pages 985--991, 2003.

\bibitem[\protect\citeauthoryear{Richardson and
  Domingos}{2004}]{richardson04markov}
Richardson, M. and Domingos, P.
\newblock Markov {L}ogic {N}etworks.
\newblock Technical report, Department of Computer Science, University of
  Washington, 2004.

\bibitem[\protect\citeauthoryear{Sanner and
  Abbasnejad}{2012}]{sanner2012symbolic}
Sanner, S. and Abbasnejad, E.
\newblock Symbolic variable elimination for discrete and continuous graphical
  models.
\newblock In {\em Proceedings of the Twenty-Sixth AAAI Conference on Artificial
  Intelligence}, 2012.

\bibitem[\protect\citeauthoryear{Taghipour \bgroup \em et al.\egroup
  }{2012}]{taghipour12lifted}
Taghipour, N., Fierens, D., den Broeck, G.~V., Davis, J., and Blockeel, H.
\newblock Lifted variable elimination: A novel operator and completeness
  results.
\newblock {\em CoRR}, abs/1208.3809, 2012.

\bibitem[\protect\citeauthoryear{{Van den Broeck} \bgroup \em et al.\egroup
  }{2011}]{vanDenBroeck11lifted}
{Van den Broeck}, G., Taghipour, N., Meert, W., Davis, J., and Raedt, L.~D.
\newblock Lifted probabilistic inference by first-order knowledge compilation.
\newblock In {\em In Proceedings of the 22nd International Joint Conference on
  Artificial Intelligence}, pages 2178--2185, 2011.

\bibitem[\protect\citeauthoryear{Zhang and Poole}{1994}]{zhang94simple}
Zhang, N.~L. and Poole, D.
\newblock A simple approach to {Bayesian} network computations.
\newblock In {\em Proceedings of the Tenth Biennial Canadian Artificial
  Intelligence Conference}, 1994.

\end{thebibliography}

\newpage
\appendix
\section{SUPPLEMENTARY MATERIALS}

\subsection{Theorems \ref{the:tuple-and-empty-set-simplifiers} and \ref{the:inversion-modulo-theories} and their proofs}

\firsttupleandemptysetsimplifiers*
\begin{proof}
	Intuitively, this theorem is analogous to an algorithm that converts propositional
	formulas into an equivalent disjunctive normal form (DNF), that is, a disjunction of conjunctive clauses (that is, conjunctions of literals).
	Once a DNF is reached, contradictory conjunctive clauses are eliminated,
	and therefore the formula is satisfiable
	if and only if there is at least one conjunctive clause with at least one literal.
	In this analogy, conjunctions and disjunctions correspond to intersection and union,
	and sets correspond to conjunctive clauses. Empty sets are eliminated, and
	if at the end we have a union of sets, and the empty ones have been eliminated,
	this means that the resulting set is not empty.
	
	Formally, the theorem is proven by induction on the \emph{distance vector},
	a tuple that measures how far
	an expression is from being solved.
	Before we define the distance vector, we need to inductive define,
	for any expression $E$, the \define{intersection-union nesting} $\nesting(E)$:
	\begin{align*}
	\nesting(E) = 
	\begin{cases}
	\qquad \;\;\sum_i \nesting(E_i), \text{ if $E = E_1 \cap \dots \cap E_n$}\\
	1 + \max_i \nesting(E_i), \text{ if $E = E_1 \cup \dots \cup E_n$}\\
	0, \text{ if $E$ is any other expression.}\\
	\end{cases}
	\end{align*}
	Intuitively, $\nesting$ measures how far we are from
	a ``flat'' union of intersections.
	
	The \define{distance vector} of an expression $E$
	is a vector of non-negative integers
	that is lexicographically ordered,
	with the most significant component listed first:
	\begin{enumerate}
		\item $\nesting(E)$;
		\item number of intensional unions ($\bigcup$);
		\item number of existential and universal quantifications;
		\item number of $\cap$ applications;
		\item sum of lengths of extensionally defined sets ($\{\dots\}$);
		\item number of $\in$ applications;
		\item number of comparisons to $\emptyset$;
		\item number of tuples.
	\end{enumerate}

	In the base case, the distance vector is a tuple of zeros,
	and therefore the expression is a formula without any tuple or set operators,
	satisfying the theorem.
	
	Otherwise, the distance vector contains at least one non-zero component.
	If we can show that there is always at least one applicable simplifier,
	and that every simplifier application results in an expression
	with a smaller distance vector with respect to the lexicographical order,
	then the theorem will have been proven by induction.
	
	There is always an applicable simplifier
	to expressions with non-zero distance vector, because
	in that case there is at least one tuple operator or
	a comparison between a set expression and $\emptyset$:
	\begin{itemize}
	\item if there is a set expression, it must be one of
	$\cap$, $\cup$ applications or $\bigcup$,
	and there is at least one simplifier for each of these;
	\item If there is a tuple anywhere, it is either inside
	a tuple comparison, or inside a set; if it is in a comparison,
	simplifier 1 applies; if it is in a set, one of the
	set simplifiers applies.
	\end{itemize}

	Once it is established that there is always an applicable simplifier,
	the next step is whether the distance vector is always decreased
	according to its lexicographical order.
	Simplifiers 1, 2, 4, 5, 6, 10, and 11 strictly decrease one or more of the distance vector
	components without increasing any other, so for expressions for which any of them
	apply, the theorem is proven by induction on the distance vector.
	
	The remaining simplifiers decrease a distance vector component while
	increasing others, but the ones increased are always less significant
	in the lexicographical order than the one decreased:
	
	\begin{itemize}
		\item Simplifier 3 decreases the number of intensional unions
		at the cost of the less significant number of existential quantifications;
		\item Simplifier 7 decreases the sum of lengths of extensionally defined
		sets at the cost of the less significant number of $\in$ applications;
		\item Simplifier 8 decreases the number of intensional unions
		at the cost of the less significant number of universal quantifications;
		\item Simplifier 9 duplicates $S_3$ and therefore
		\emph{doubles} all distance vector components in $S_3$,
		with the exception the most significant one, \nesting,
		which is decreased:
		\begin{align*}
		&\nesting((S_1 \cup S_2) \cap S_3 = \emptyset) \\ 
		& = \max(\nesting((S_1 \cup S_2) \cap S_3), 0) \\
		& = \nesting((S_1 \cup S_2) \cap S_3) \\
		& = \nesting(S_1 \cup S_2) + \nesting(S_3) \\
		& = 1 + \max(\nesting(S_1), \nesting(S_2)) + \nesting(S_3) \\
		& = 1 + \max(\nesting(S_1) + \nesting(S_3),\\
		& \hspace{1.9cm}  \nesting(S_2) + \nesting(S_3)) \\
		& = 1 + \max(\nesting(S_1 \cap S_3),\\
		& \hspace{1.9cm}  \nesting(S_2 \cap S_3)) \\
		& = 1 + \nesting((S_1 \cap S_3) \cup (S_2 \cap S_3)) \\
		& > \quad \;\;\,\nesting((S_1 \cap S_3) \cup (S_2 \cap S_3)) \\
		& = \quad \;\;\,\nesting((S_1 \cap S_3) \cup (S_2 \cap S_3) = \emptyset).
		\end{align*}
	\end{itemize}
	To summarize, we have shown that, for every expression in the language
	of interest, there is always an applicable simplifier,
	and that all simplifiers decrease the distance vector until it reaches
	the base case all-zero distance vector, which is free of tuple and set operators.
\end{proof}

\firstinversion*
\begin{proof}
	Let $m$ be the number of possible assignments to $x_1,\dots,x_m$.
	We prove the theorem by induction on $m$.
	If $m$ is $0$,
\begin{align*}
& \bigoplus_{f \in A \rightarrow B} \quad \bigotimes_{x_1 \in T_1 : C_1} \dots \bigotimes_{x_k \in T_k : C_k} E
\\& = \bigoplus_{f \in \emptyset \rightarrow B} \quad 1
\\& = \bigotimes_{x_1 \in T_1 : C_1} \dots \bigotimes_{x_k \in T_k:C_k} \quad \bigoplus_{f \; \in \; \emptyset \rightarrow B} 1
\\& = \bigotimes_{x_1 \in T_1 : C_1} \dots \bigotimes_{x_k \in T_k:C_k} \quad \bigoplus_{f \; \in \; \emptyset \rightarrow B} E
\\& = \bigotimes_{x_1 \in T_1 : C_1} \dots \bigotimes_{x_k \in T_k:C_k} \quad \bigoplus_{f \; \in \; oc_f[E] \rightarrow B} E,
\end{align*}
because the empty products allow the substitution of $1$ by $E$ and $\emptyset$ by $oc_f[E]$
without change.

If $m > 0$, let $\bar{x}$ be the first possible assignment to $x_1,\dots,x_k$
satisfying $C_1 \wedge \dots \wedge C_k$.
Then
\begin{align*}
& \bigoplus_{f \in A \rightarrow B} \quad \bigotimes_{x_1 \in T_1 : C_1} \dots \bigotimes_{x_k \in T_k : C_k} E
\\& = \text{(separating $\bar{x}$ from other assignments)}
\\& \bigoplus_{f_1 \in oc_f[E][x_1,\dots,x_k/\bar{x}] \rightarrow B} 
\\& \bigoplus_{f \in (A \setminus oc_f[E][x_1,\dots,x_k/\bar{x}]) \rightarrow B}
\\& \hspace{1cm} E_1 \qquad \otimes
\\& \hspace{1cm} \bigotimes_{(x_1,\dots,x_k) \in T_1 \times \dots \times T_k : C_1 \wedge \dots \wedge C_k \wedge (x_1,\dots,x_k) \neq \bar{x}}  E,
\\& \qquad \text{where $E_1 = E[f/f_1]$}
\end{align*}
\begin{align*}
& = \text{($E_1$ has no occurrences of $f$)}
\\& \Bigl( \bigoplus_{f_1 \in oc_f[E][x_1,\dots,x_k/\bar{x}] \rightarrow B} E_1 \Bigr) \qquad \otimes 
\\& \bigoplus_{f \in (A \setminus oc_f[E][x_1,\dots,x_k/\bar{x}]) \rightarrow B}
\\& \bigotimes_{(x_1,\dots,x_k) \in T_1 \times \dots \times T_k : C_1 \wedge \dots \wedge C_k \wedge (x_1,\dots,x_k) \neq \bar{x}}  E
\end{align*}
\begin{align*}
& = \text{(by induction on $m$)}
\\& \Bigl( \bigoplus_{f_1 \in oc_f[E][x_1,\dots,x_k/\bar{x}] \rightarrow B} E_1 \Bigr) \qquad \otimes 
\\& 
\bigotimes_{(x_1,\dots,x_k) \in T_1 \times \dots \times T_k : C_1 \wedge \dots \wedge C_k \wedge (x_1,\dots,x_k) \neq \bar{x}} \bigoplus_{f \in oc_f[E] \rightarrow B}  E
\end{align*}
\begin{align*}
& = \text{(renaming $f_1$ to $f$ and using the fact that $E_1 = E[f/f_1]$)}
\\& \Bigl( \bigoplus_{f \in oc_f[E][x_1,\dots,x_k/\bar{x}] \rightarrow B} E \Bigr) \qquad \otimes 
\\& 
\bigotimes_{(x_1,\dots,x_k) \in T_1 \times \dots \times T_k : C_1 \wedge \dots \wedge C_k \wedge (x_1,\dots,x_k) \neq \bar{x}} \bigoplus_{f \in oc_f[E] \rightarrow B}  E
\end{align*}
\begin{align*}
& = \text{(introducing intensional products on $x_1,\dots,x_k$ bound to $\bar{x}$)}
\\& \bigotimes_{(x_1,\dots,x_k) \in T_1 \times \dots \times T_k : C_1 \wedge \dots \wedge C_k \wedge (x_1,\dots,x_k) = \bar{x}}
\\&\hspace{2cm}\Bigl( \bigoplus_{f \in oc_f[E][x_1,\dots,x_k/\bar{x}] \rightarrow B} E \Bigr) \qquad \otimes 
\\& \bigotimes_{(x_1,\dots,x_k) \in T_1 \times \dots \times T_k : C_1 \wedge \dots \wedge C_k \wedge (x_1,\dots,x_k) \neq \bar{x}} \bigoplus_{f \in oc_f[E] \rightarrow B}  E
\end{align*}
\begin{align*}
& = \text{(merging $\bigotimes$ by disjuncting their constraints)}
\\& \bigotimes_{(x_1,\dots,x_k) \in T_1 \times \dots \times T_k : C_1 \wedge \dots \wedge C_k \wedge \bigl((x_1,\dots,x_k) = \bar{x} \vee (x_1,\dots,x_k) \neq \bar{x}\bigr)}
\\& \hspace{5cm} \bigoplus_{f \in oc_f[E] \rightarrow B}  E
\end{align*}
\begin{align*}
& = \text{(eliminating tautology on $\bar{x}$ and separating $\bigotimes$ per index)}
\\& \bigotimes_{x_1 \in T_1 : C_1} \dots \bigotimes_{x_k \in T_k:C_k} \quad \bigoplus_{f \; \in \; oc_f[E] \rightarrow B} E
\end{align*}
The final expression is $O(2^{\prod_i |\{x_i \in T_i : C_i\}|})$ cheaper to
evaluate because the final $f$ has all $x_i$ bound to a single value,
mapping each of their $\prod_i |\{x_i \in T_i : C_i\}|$ assignments to
a single one and thus dividing the total size of the domain
over which one must iterate.
\end{proof}
\end{document}